\documentclass{article}

    \usepackage[preprint]{neurips_2025}

\usepackage[utf8]{inputenc} 
\usepackage[T1]{fontenc}    
\usepackage{hyperref}       
\usepackage{url}            
\usepackage{booktabs}       
\usepackage{amsfonts}       
\usepackage{nicefrac}       
\usepackage{microtype}      
\usepackage{xcolor}         

\usepackage{array}
\usepackage{tabularx}

\newif\ifshowcomments
\showcommentsfalse

\newcommand{\nadav}[1]{%
    \ifshowcomments%
        {\color{orange} [NADAV: #1]}%
    \fi%
}

\usepackage{amsmath}
\usepackage{amssymb}
\usepackage{mathtools}
\usepackage{amsthm}
\usepackage{thmtools}
\usepackage{thm-restate}
\usepackage[ruled,vlined,algo2e]{algorithm2e}  
\usepackage[capitalise]{cleveref}       
\usepackage{verbatim}

\declaretheorem[style=plain,numberwithin=section]{theorem}

\declaretheorem[style=plain,sibling=theorem]{lemma}

\declaretheorem[style=definition,sibling=theorem]{definition}
\declaretheorem[style=definition,sibling=theorem]{assumption}

\crefname{section}{§}{§}

\title{Out-of-Vocabulary Sampling Boosts Speculative~Decoding}

\author{%
  Nadav Timor\thanks{\texttt{nadav.timor@weizmann.ac.il}} \\
  Weizmann Institute of Science \\
  \And
  Jonathan Mamou \\
  Intel Labs \\
  \And
  Oren Pereg \\
  Intel Labs \\
  \And
  Hongyang Zhang \\
  University of Waterloo \\
  \And
  David Harel \\
  Weizmann Institute of Science \\
}

\begin{document}

\maketitle

\begin{abstract}
Speculative decoding relies on fast and accurate drafters. Recent state-of-the-art language models employ larger and larger vocabularies, which significantly slows down drafters. One promising approach to boost the efficiency of speculative decoding is to use drafters with smaller vocabularies. However, existing sampling methods cannot draw out-of-vocabulary tokens, creating a tradeoff between drafters' vocabulary size and acceptance rates. This paper introduces \emph{Redistributing~Drafter~Kernels~(RDK)}, the first out-of-vocabulary sampler that effectively recovers acceptance rates by virtually restoring pruned target tokens. RDK leverages token-affinity priors to reallocate drafter mass towards high-overlap regions. We prove mathematically that RDK can achieve higher acceptance rates than vanilla and state-of-the-art samplers. We provide an efficient first-order approximation of RDK and prove that it reduces redistribution times from $\mathcal{O}(N^2)$ to $\mathcal{O}(N)$, enabling lightweight implementations for large vocabularies. Our experiments demonstrate that this linear-time RDK significantly boosts acceptance rates even after \emph{extreme~pruning} (removing more than 75\% of the drafter's vocabulary), where existing samplers fail. RDK opens the door to extremely pruned drafters, which were previously impractical.
\end{abstract}

\section{Introduction}

Speculative decoding has gained substantial attention in recent years for its ability to accelerate the inference of large language models without compromising output quality~\citep{leviathan2023fast, chen2023accelerating, miao2024specinfer, sun2025block, timor2025accelerating}. By exploiting the parallelism of modern GPUs, it mitigates the memory and latency bottlenecks associated with standard autoregressive decoding. The approach involves two models: a draft model, which proposes multiple future tokens, and a target model, which verifies these predictions in parallel using rejection sampling—ensuring that quality remains intact.

To speed up a target model, two primary strategies are used to choose the draft model. The first involves training a dedicated drafter at the feature level. Notable examples include Medusa~\cite{cai2024medusa} and the EAGLE series~\citep{li2024eagle1, li2024eagle2, li2025eagle3}. However, this approach can be costly and model-specific. The second strategy leverages an off-the-shelf smaller language model—typically the smallest in the same family (e.g., LLaMA)—as the drafter, benefiting from shared vocabulary and architecture. Yet, this method becomes infeasible when the target model is already the smallest in its family, leaving no smaller variant to serve as a compatible drafter.

In this paper, we explore whether speculative decoding remains effective when the draft and target models do not share the same vocabulary. A straightforward workaround is to use the intersection of their vocabularies: if a drafted token belongs to this shared subset, it is passed to the target model for verification; otherwise, it is discarded. This method, referred to as Token-Level Intersection (TLI), has been previously studied, but it often suffers from low acceptance rates—especially when the vocabulary overlap is small.

We address this limitation by focusing on the case of pruned or small-intersection drafters, whose inability to propose out-of-vocabulary (OOV) tokens significantly reduces their utility. To overcome this challenge, we propose Redistributing Drafter Kernels (RDK)—the first OOV-aware sampling strategy that effectively restores acceptance rates by virtually recovering pruned target tokens. RDK leverages token-affinity priors to redistribute the probability mass of the drafter toward tokens that are more likely to be accepted, thereby increasing the yield of valid proposals without incurring additional forward latency. Our contributions are as follows:
\begin{itemize}
    \item \textbf{First OOV sampler.} We introduce RDK, the first sampling method capable of drawing out-of-vocabulary tokens by redistributing drafter kernels.
    \item \textbf{Theoretical guarantees.} In a synthetic setting, we prove that RDK achieves higher acceptance rates than both vanilla and state-of-the-art samplers.
    \item \textbf{Efficient approximation.} We derive a first-order approximation that reduces redistribution complexity from $\mathcal{O}(N^2)$ to $\mathcal{O}(N)$, enabling lightweight implementations for large vocabularies.
    \item \textbf{Extreme-pruning performance.} We demonstrate empirically that linear-time RDK maintains high acceptance even after extreme pruning (removing over 75\% of the drafter's vocabulary), where existing samplers fail.
\end{itemize}

The rest of the paper is organized as follows: \cref{sec:prelim} presents formal definitions and preliminaries. \cref{sec:related-work} reviews related work on speculative decoding, vocabulary pruning, and sampling methods. \cref{sec:method} describes RDK and its efficient approximation. \cref{sec:theoretical-results} provides theoretical analyses. \cref{sec:empirical-results} reports empirical evaluations. Finally, \cref{sec:discussion} concludes the paper.

\section{Preliminaries}\label{sec:prelim}
\nadav{
    Add analysis for smaller LM heads (like EAGLE). See Kindle.
}


\citet{timor2025distributed} extensively studied how the effectiveness of speculative decoding \citep{leviathan2023fast,chen2023accelerating,miao2024specinfer,sun2025block,timor2025accelerating} is governed by the selection of the drafter model. Their analysis shows the expected speedup of speculative decoding for any given drafter, where the drafters are characterized by two factors: (i) \emph{forward latency} (i.e., the average time required to compute their forward pass on an average input), and (ii) \emph{acceptance rate} (\cref{def:acceptance-rate}). As drafters get faster or more accurate, the decoding gets more efficient because each iteration of speculative decoding is faster and yields more new tokens, respectively.

\begin{definition}[\citep{leviathan2023fast,chen2023accelerating}]\label{def:acceptance-rate}
    Let \(p\) and \(q\) be two probability distributions over a finite vocabulary \(\mathcal{V}\). The \emph{\textbf{acceptance rate}} \(\alpha\) is defined as $\alpha~:=~\sum_{x \in \mathcal{V}} \min\{p(x),q(x)\}$.
\end{definition}
Intuitively, \(\alpha\) measures the total overlap between the two distributions. In speculative decoding, \(p\) is the \emph{target} distribution that the \emph{drafter} \(q\) estimates. As the overlap between the drafter distribution \(q\) and the target distribution \(p\) increases, the acceptance rate increases, approaching \(1\). The \emph{drafter~kernel} represents the portion of probability mass that the drafter assigns to tokens where \(q(x) > p(x)\). Such assignments do not contribute to increasing the acceptance rate. This excess can be viewed as redundant, since its contribution to the acceptance rate is zero.

\begin{definition}\label{def:drafter-kernel}
    The \emph{\textbf{drafter kernel}} \(k\) is defined as \(k(x) := \max(0, q(x) - p(x))\).
\end{definition}

Note that redistributing the probability mass of the drafter kernel $k$ cannot decrease the acceptance rate. In fact, minimizing $\|k\|_1$ is equivalent to maximizing the acceptance rate (\cref{lemma:minimizing-kernels-or-l1-increases-acceptance-rate}).

\begin{restatable}{lemma}{MinimizingKernelsOrLOneIncreasesAcceptanceRate}\label{lemma:minimizing-kernels-or-l1-increases-acceptance-rate}
    Let \(\alpha\) be the acceptance rate from \cref{def:acceptance-rate} and \(k\) be the drafter kernel from \cref{def:drafter-kernel}, where \(p\) and \(q\) are probability distributions over a finite set. Then,
    \[
    \alpha = 1 - \|k\|_1 = 1 - \frac{1}{2}\|p - q\|_1.
    \]
    In particular, minimizing \(\|k\|_1\) or $\|p - q\|_1$ increases the acceptance rate \(\alpha\).
\end{restatable}
\begin{proof}
    See \cref{appendix:proof-minimizing-kernels}.
\end{proof}

This paper reveals how existing state of the art decoding methods can increase acceptance rates by redistributing drafter kernels.




\paragraph{Notation.}
Let the total event space be of size $m$, indexed by $\{1, \dots, m\}$.
Let $p \in \mathbb{R}^m$ be a \emph{target} probability vector supported on a subset $T \subseteq \{1, \dots, m\}$. 
Let $q \in \mathbb{R}^m$ be a \emph{draft} probability vector supported on a subset $D \subseteq \{1, \dots, m\}$.
Denote the intersection of the supports by $\mathcal{I} := T \cap D$ and assume that $\mathcal{I} \ne \emptyset$.
Denote the \emph{out-of-vocabulary} tokens in $q$ by $\mathcal{F} := D \setminus \mathcal{I}$.

\section{Related work}\label{sec:related-work}

\paragraph{$D=T$ (homogeneous).}
Vanilla speculative decoding algorithms operate on homogeneous drafters, that is, drafters for which $\mathcal{F} = \emptyset$ \citep{leviathan2023fast, chen2023accelerating}.

\paragraph{$D \not\subseteq T$ (heterogeneous).}
 Generalized speculative decoding algorithms for heterogeneous drafters (where $\mathcal{F} \ne \emptyset$), has been recently introduced by \citet{timor2025accelerating} and their Token-Level Intersection (TLI) algorithm have become the state-of-the-art, widely adopted in practice. We denote by $q_0 \in \mathbb{R}^m$ the probability vector you get by applying the TLI redistribution over the drafter $q$, which is defined as follows.
\begin{equation}\label{eq:tli-drafter}
q_0(i) :=
\begin{cases}
\displaystyle \frac{q(i)}{\sum_{j \in \mathcal{I}} q(j)}, & \text{if } i \in \mathcal{I}; \\[2mm]
0, & \text{otherwise.}
\end{cases}
\end{equation}
That is, zeroing out all out-of-vocabulary mass from $q$ (i.e., setting $q(i) = 0$ for all $i \in \mathcal{F}$) and then normalizing over the intersection $\mathcal{I}$ (see \cref{alg:tli} in the appendix for the full TLI algorithm).

\paragraph{\boldmath$D \subset T$ (pruning).} \citet{zhao2025fr} proposed a method for pruning the drafter's vocabulary based on token frequency in the training data. Their approach is motivated by the observation that token frequency distributions in large-scale datasets are highly imbalanced, exhibiting a long-tailed structure. For example, in Llama-3-8B \citep{grattafiori2024llama} evaluated on 1B tokens randomly sampled from the SlimPajama-627B dataset \citep{cerebras2023slimpajama}, 75\% of vocabulary tokens account for less than 5\% of all token occurrences. By removing infrequent tokens from the drafter's vocabulary, they achieve significant decoding speedups on state-of-the-art models, including Llama-3, Llama-3.2, and Qwen-2 \citep{yang2024qwen2}. The method works as follows: given a drafter's tokenizer and a dataset $S$, they compute the empirical frequency of each vocabulary token over $S$ and retain only the most frequent tokens. For a formal description of their algorithm, see \cref{alg:fr-prune} in the appendix.

\section{Redistributing Drafter Kernels (RDK)}\label{sec:method}

Below we introduce a new algorithm that provably increases the acceptance rate over TLI for any drafter $q$ by using prior knowledge about the target distribution $p$. The matrix $M \in \mathbb{R}^{m \times m}$ holds the prior knowledge about the target distribution $p$. \cref{assumption:m-is-row-stochastic} ensures that the redistribution operation in \cref{alg:rdk} preserves the total probability mass.

\begin{assumption}\label{assumption:m-is-row-stochastic}
The matrix $M \in \mathbb{R}^{m \times m}$ is row-stochastic, i.e., the entries of any row are non-negative and sum to $1$.
\end{assumption}

\begin{algorithm2e}[H]
    \caption{Redistributing Drafter Kernels (RDK)}
    \label{alg:rdk}
    \KwIn{Probability distributions $p$ and $q$ over vocabularies $T$ and $D$, respectively. Row-stochastic matrix $M \in \mathbb{R}^{m \times m}$ (\cref{assumption:m-is-row-stochastic}).}
    \nl If $D \setminus \mathcal{I} \ne \emptyset$, apply TLI to redistribute the probability mass of the drafter $q$ into the intersection $\mathcal{I} := T \cap D$ and denote the resulting TLI distribution $q'$ where $q'(t) = 0$ for all $t \notin T$ and $\sum_{t \in \mathcal{I}} q'(t) = 1$\;
    \nl Redistribute the mass of $q'$ into the target support $T$ via $p' := M^\top q'$ such that $\sum_{t \in T} p'(t) = 1$\;
    \nl \textbf{Run} vanilla speculative decoding (\cref{alg:sd}) with $p, p'$\;
\end{algorithm2e}

The RDK sampler (\cref{alg:rdk}) consists of two steps: applying TLI (\cref{alg:tli} in the appendix) followed by an out-of-vocabulary redistribution. It is easy to see that if $M$ is the identity matrix, then RDK reduces to TLI.

\paragraph{Token-affinity priors.}
RDK relies on token-affinity priors, represented by the matrix $M$ from \cref{alg:rdk}.
We can construct $M$ based on prior knowledge about \emph{co-occurrence} of tokens in the target distribution by estimating it on training data.



To simplify, we first show show to construct the token-affinity matrix $M$ from the covariance matrix of the target distribution $\Omega~\in~\mathbb{R}^{|T| \times |T|}$. In practice, we estimate the covariance matrix $\Omega$ from data. When following such a construction of $M$, we say that $M$ is the \emph{co-occurrence} matrix of the target distribution. In the literature, matrices similar to $M$ are sometimes named transition matrices, Markov kernels, or affinity matrices.

Given a covariance matrix $\Omega \in \mathbb{R}^{|T| \times |T|}$, we define the token-affinity matrix $M$ via a softmax-based transformation:
\begin{align}\label{eq:symbolic-m}
    M_{ij} = \frac{\exp(\Omega_{ij}/\tau)}{\sum_{k \in T} \exp(\Omega_{ik}/\tau)} \quad \text{for } i \in \mathcal{I},\, j \in T,
\end{align}
where $\tau$ is a temperature parameter that controls the sharpness of the redistribution.

The token-affinity matrix \( M \) is calculated once at preprocessing time, and it can be loaded into GPU memory once, also at preprocessing time. The projection function \( M \) is then applied to the draft model's logits \( q \) at inference time to obtain the new probability distribution \( q' \) over the target vocabulary \( T \). This projection method is training-free and does not require any additional data or fine-tuning of drafters. It is a general method that can be applied to any pair of vocabularies, regardless of their sizes or tokenization schemes.

\nadav{
    Consider adding analysis for smaller LM heads (like EAGLE). See Kindle.
}

\paragraph{Efficient first-order Taylor approximation.}
The full RDK update (\cref{alg:rdk}) computes 
\(
p_{\mathrm{RDK}} = M^\top q',
\)
where \(q'=\hat q/\|\hat q\|_{1}\) is the drafter distribution pruned to \(\mathcal{I}\) and renormalized, and \(M\in\mathbb R^{N\times N}\) is the row-stochastic co-occurrence matrix from \cref{eq:symbolic-m}.  Forming and applying \(M\) costs \(\mathcal O(N^2)\), which is impractical for large vocabularies.  Instead, we derive a closed-form, first-order Taylor expansion of \(M^\top q'\) that runs in \(\mathcal O(N)\) time and requires no explicit matrix.

Let \(\theta = p^\top q'\).  Define
\begin{align}\label{def:rdk-taylor}
\tilde p_i \;=\;\frac{N\,q'_i + \theta\,p_i}{\,N + p_i\,},
\qquad
p_{\mathrm{Taylor}} \;=\;\frac{\tilde p}{\sum_j\tilde p_j}.
\end{align}

\cref{thm:rdk-taylor} proves that the linear-time approximator \(p_{\mathrm{Taylor}}\) from \cref{def:rdk-taylor} coincides with the exact, full RDK \(M^\top q'\) from \cref{alg:rdk} up to first order in \(\theta\). In other words, we show that the approximation is accurate.

\section{Theoretical results}\label{sec:theoretical-results}

\subsection{Acceptance recovery}

\cref{theorem:bounded-l1} upper bounds the error of the RDK drafter $p'$ from \cref{alg:rdk}. It shows that, in the worst case, the error is at most $\|M^\top p - p\|_1$ larger than the error of the TLI drafter $q_0$ from \cref{eq:tli-drafter}.
By \cref{lemma:minimizing-kernels-or-l1-increases-acceptance-rate}, this small $L_1$ distance of RDK drafters is equivalent to a larger acceptance rate.

\begin{restatable}{theorem}{boundedlone}\label{theorem:bounded-l1}
  Let $p, q, q_0, M, p'$ be defined in \cref{alg:rdk}. Then we have that
  \[
  \boxed{
  \|p' - p\|_1 \;\le\; \|q_0 - p\|_1 + \epsilon
  }
  \]
  where $\epsilon$ is a constant such that $\|M^\top p - p\|_1 \le \epsilon$.
  That is, the RDK drafter $p'$ is either closer to the target distribution $p$ or at least as close as the TLI drafter $q_0$ of \citet{timor2025accelerating}, with respect to $L_1$ distance.
\end{restatable}

\begin{proof}
    See \cref{appendix:proof-bounded-l1}.
\end{proof}

It is easy to see that if $M$ is the identity matrix, then \cref{alg:rdk} reduces to TLI (\cref{alg:tli} in the appendix), and \cref{theorem:bounded-l1} becomes trivial.

To further analyze the effectiveness of redistributing drafter kernels, we posit a theoretical model in which the logits produced by a language model are assumed to be sampled from a multivariate Gaussian distribution, namely,
\begin{equation}\label{def:logits}
z \sim \mathcal{N}(\mu,\,\Omega),
\end{equation}
where \( z \in \mathbb{R}^m \) is the vector of logits over an \( m \)-token vocabulary, \(\mu \in \mathbb{R}^m\) is the mean vector, and \( \Omega \in \mathbb{R}^{m \times m} \) is the covariance matrix.
Once we have the Gaussian sample \( z \), the conversion to a valid probability distribution over tokens is performed via the softmax function:
\[
p_i = \frac{\exp(z_i)}{\sum_{j=1}^m \exp(z_j)} \quad \text{for } i = 1, \ldots, m.
\]
This transformation yields a probability vector \( p \in \Delta^{m-1} \) where \( \Delta^{m-1} \) denotes the \((m-1)\)-simplex.

\cref{thm:softmaxstability} upper bounds \( \|q - p\|_1 \) in terms of the approximation error of the logits. It suggests that any approximation error in the logits translates to only a half (or less) of the error in the probability distribution. Therefore, by \cref{thm:softmaxstability}, the slight perturbations introduced by our parametric approximations do not compromise the validity of the empirical results in \cref{sec:empirical-results}, which are based on synthetic target distributions.

\begin{restatable}{theorem}{SoftmaxStability}\label{thm:softmaxstability}
    Let \( z \in \mathbb{R}^m \) be a vector of logits, and let \( \eta \in \mathbb{R}^m \) be a noise vector with \( \|\eta\|_1 \le \epsilon \). Define $p := \mathrm{softmax}(z)$ and $q := \mathrm{softmax}(z + \eta)$.
    Let \( p = \mathrm{softmax}(z) \) and \( q = \mathrm{softmax}(z + \eta) \).
    Then $\|q - p\|_1 \le \frac{1}{2} \epsilon$.
\end{restatable}

\begin{proof}
    See \cref{appendix:proof-softmaxstability}.
\end{proof}

The following theorem shows the advantage of RDK over TLI or the vanilla sampler in terms of acceptance rate.
\begin{restatable}{theorem}{RDKAcceptance}\label{thm:rdk-acceptance}
    Let $m=2$. In \cref{def:logits}, denote by
    \begin{equation*}
        \mu =
        \begin{bmatrix}
        \mu_1\\
        \mu_2
        \end{bmatrix}
        \quad
        \text{and}
        \quad
        M=
        \begin{bmatrix}
            M_{11} & M_{12}\\
            M_{12} & M_{22}
        \end{bmatrix}.
    \end{equation*}
    Assume that $\mu_1+3\sqrt{M_{11}}\le\mu_2-3\sqrt{M_{22}}$ (that is, $\mu_1\ll\mu_2$). Consider the notations in \cref{alg:rdk}, where $q'=[1,0]^\top$ (i.e., the output probability by TLI without redistribution). Then with probability 99.46\%, under \cref{assumption:m-is-row-stochastic}, we have
    \begin{equation*}
        \alpha_\text{TLI}\leq \alpha_{\text{RDK}}.
    \end{equation*}
\end{restatable}

\begin{proof}
    See \cref{appendix:proof-rdk-acceptance}.
\end{proof}

\subsection{Accurate linear-time approximation}

The next theorem shows that \(p_{\mathrm{Taylor}}\) from \cref{def:rdk-taylor} coincides with the exact, full RDK \(M^\top q'\) from \cref{alg:rdk} up to first order in \(\theta\).

\begin{restatable}{theorem}{TaylorRDKApprox}
\label{thm:rdk-taylor}
Under \cref{assumption:m-is-row-stochastic}, let \(p,\hat q\in\Delta^{V-1}\), \(T\subseteq\{1,\dots,V\}\), \(q'=\hat q/\|\hat q\|_1\), \(\theta=p^\top q'\), \(N=|T|\), \(M\) the co-occurrence matrix from \cref{eq:symbolic-m}, and \(p_{\mathrm{Taylor}}\) from \cref{def:rdk-taylor}. Then there are constants \(C_1,C_2\) (independent of \(\theta,N\)) such that
\[
\bigl\|\,p_{\mathrm{Taylor}} - M^\top q'\bigr\|_{1}
\;\le\; C_1\,\theta \;+\; C_2\,\frac1{N^2}
= \mathcal{O}\!\Bigl(\theta + \tfrac{1}{N^2}\Bigr).
\]
\end{restatable}

\begin{proof}
    See \cref{appendix:proof-rdk-taylor}.
\end{proof}

\section{Empirical results}\label{sec:empirical-results}

\nadav{Reproduce FP-Spec. In addition to calculating the average frequency, calculate the logits' average and standard deviation per token. Then, based on the average frequency and std, redistribute excessive probability mass as follows.
}

We validate our methods and the theoretical framework through a series of complementary empirical studies.
All our code will be open-sourced upon publication.

\subsection{Underlying distributions}\label{sec:underlying-distributions}

\paragraph{Tokens.}
We replicate and extend the FR-Spec study \citep{zhao2025fr} on \texttt{llama-3.1-8b} by computing token frequencies over the first 100M tokens of the \texttt{wikitext} dataset. This analysis reveals the long-tailed nature of the \texttt{llama-3.1-8b} vocabulary and informs our choice of pruning thresholds in subsequent experiments. The result shows that 12.6\% of the tokens (16.1k from 128k) support 95\% of the probability mass. Such long-tailed distributions highlight the potential of extreme pruning, as 87.4\% of the tokens are rarely used and support only 5\% of the probability mass.

\begin{figure}[htbp]
    \centering
    \includegraphics[width=0.8\textwidth]{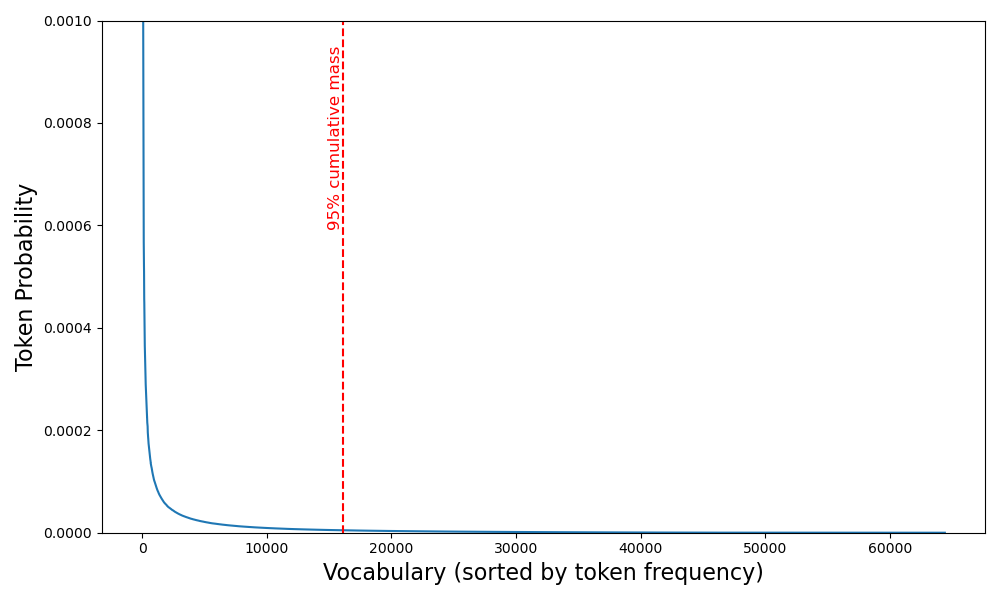}
    \caption{Token frequency analysis. 12.6\% of the tokens (16.1k from 128k) support 95\% of the cumulative probability mass of \texttt{llama-3.1-8b} on 100M tokens from the \texttt{wikitext} dataset. Such long-tailed distributions highlight the potential of extreme pruning.}
    \label{fig:token-frequency}
\end{figure}

\paragraph{Logits.}
We examine the empirical distribution of pre-softmax logits produced by \texttt{gpt-2}. 
\cref{fig:logits-distribution} shows the statistical properties of GPT-2 logits across token positions, analyzed on chunks of tokens that we randomly sampled from the \texttt{wikitext} dataset. These distributions reveal the heavy-tailed nature of logit values, with mean values predominantly negative and standard deviations typically ranging from 2-4. Notably, even the maximum logit values are predominantly negative.

\begin{figure}[htbp]
    \centering
    \includegraphics[width=0.8\textwidth]{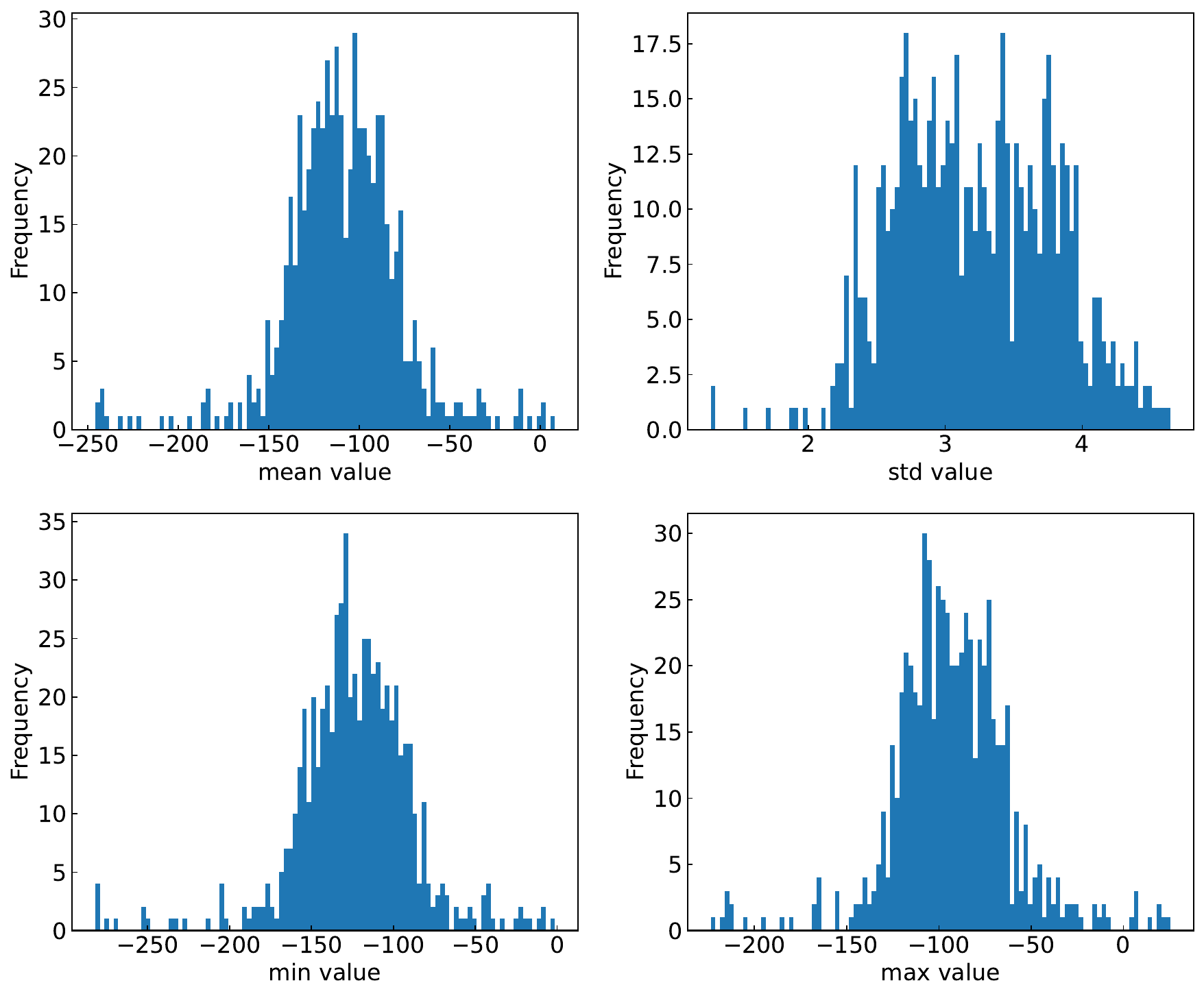}
    \caption{Histograms of \texttt{gpt-2} logit statistics across token positions. (Top-left) Mean logits, (Top-right) logit standard deviations, (Bottom-left) minimum logits, and (Bottom-right) maximum logits. Mean and minimum values concentrate in the strongly negative range, standard deviations lie between approximately 2 and 4, and maximum logits exhibit a heavy tail up to around 0, indicating that although most logits are negative with limited variability, some contexts produce comparatively higher logit scores.}
    \label{fig:logits-distribution}
\end{figure}

To assess the heavy-tailed nature of our empirical logits, we employ a quantile-quantile (Q-Q) plot\footnote{\url{https://en.wikipedia.org/wiki/Q\%E2\%80\%93Q_plot}}.  A Q-Q plot compares the ordered sample values against the theoretical quantiles of a chosen reference distribution.  Points that lie along the 45° identity line indicate close agreement in location and scale, while systematic departures reveal specific mismatches.
This visual check is especially useful when fitting heavy-tailed distributions such as the Student \(t\).
\cref{fig:logits-qq-t-distribution} shows the resulting Q-Q plot of our observed logits versus a fitted Student \(t\) distribution. It suggests that a Student \( t \) distribution closely matches the heavy tails of the observed logits.

\begin{figure}[htbp]
    \centering
    \includegraphics[width=0.8\textwidth]{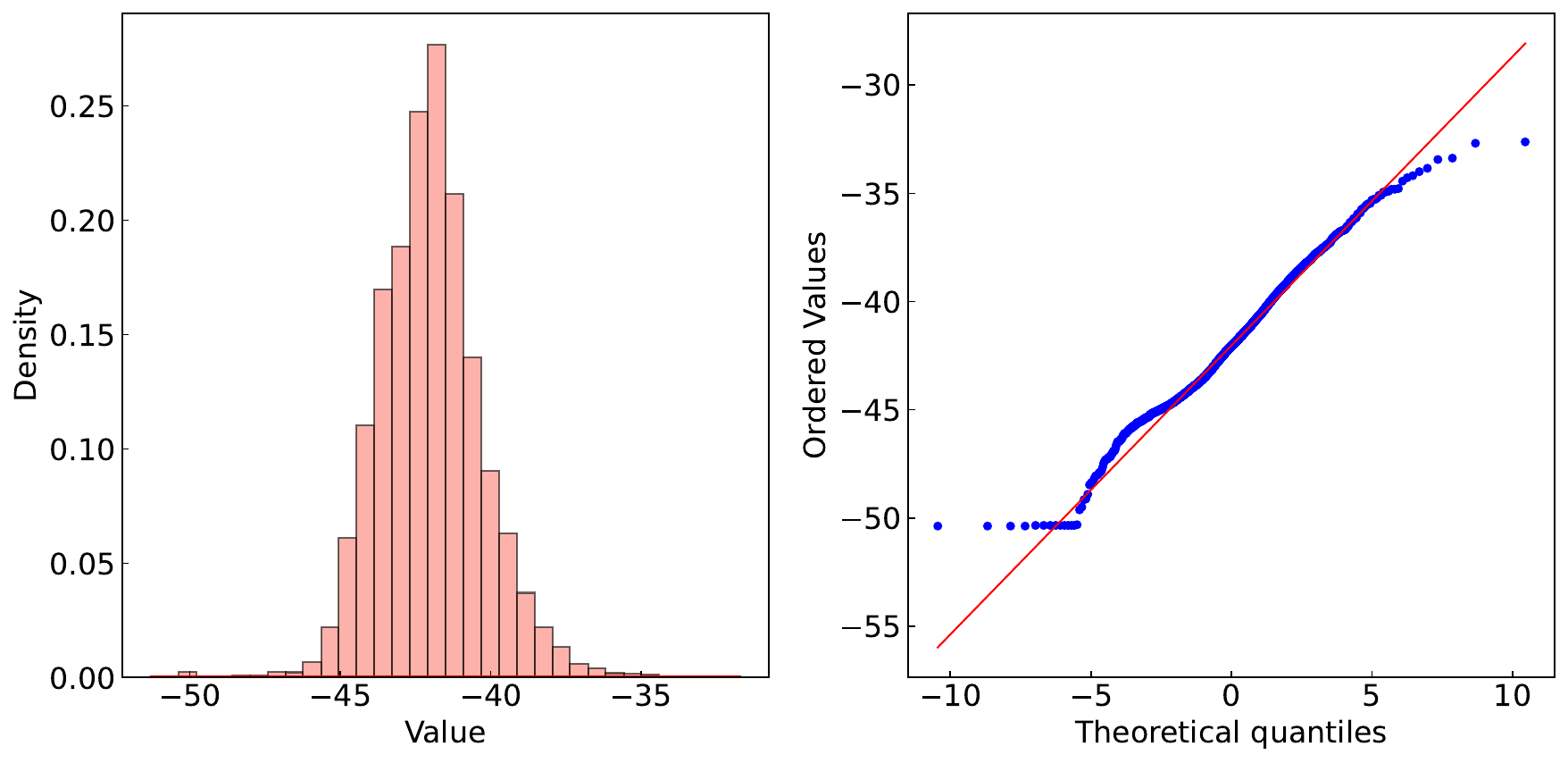}
    \caption{Empirical distribution of logits for a random prompt compared to a fitted Student\,$t$ distribution (df=5). (Left) Histogram of the observed logits overlaid with the maximum-likelihood estimated $t$-distribution PDF. (Right) Q-Q plot comparing empirical quantiles (blue) to theoretical quantiles (red), showing excellent tail alignment. 
The close match in the heavy tails motivates our use of $t$-based priors in synthetic evaluations.}
    \label{fig:logits-qq-t-distribution}
\end{figure}

\cref{sec:logits-qq-normal-distribution} shows the Q-Q plot of the logits distribution with a normal distribution, from which we learn that the logits are more likely to be distributed as the Student \(t\) distribution rather than the normal distribution, as we hypothesized in \cref{sec:theoretical-results}.

\paragraph{Output probabilities.}
We investigate the empirical distribution of post-softmax probabilities produced by \texttt{gpt-2}. \cref{fig:output-probabilities-distribution} presents histograms of per-position mean, standard deviation, minimum, and maximum probabilities, computed over random chunks sampled from \texttt{wikitext}. These plots reveal that output probabilities are highly concentrated: mean and minimum values cluster near zero, standard deviations lie in the 0.001-0.004 range, and maximum values exhibit a heavy tail spanning roughly 0.2 to 1.0. This behavior underscores that only a small subset of tokens captures most of the probability mass in any given context. This finding is consistent with the token frequency analysis in \cref{fig:token-frequency}, where only a small subset of tokens support 95\% of the probability mass.

\begin{figure}[htbp]
    \centering
    \includegraphics[width=0.8\textwidth]{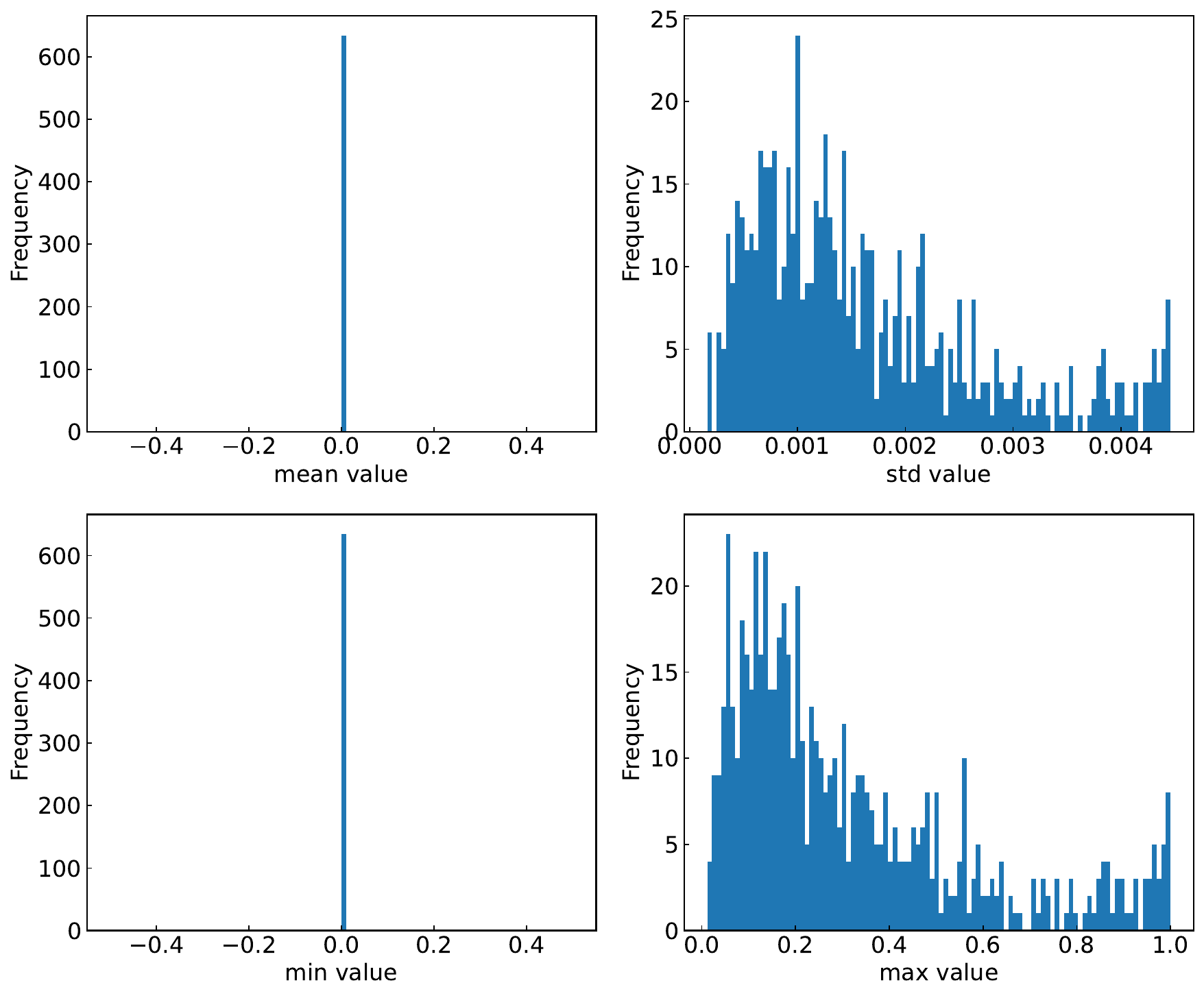}
    \caption{Histograms of \texttt{gpt-2} output-probability statistics across token positions. (Top-left) Mean probabilities, (Top-right) standard deviations, (Bottom-left) minimum probabilities, and (Bottom-right) maximum probabilities. Mean and minimum values concentrate at zero, standard deviations lie in the \(10^{-3}\) range, and maximum probabilities show a heavy tail up to 1.0, indicating sparse, high-confidence peaks amid largely negligible mass assignments.}
    \label{fig:output-probabilities-distribution}
\end{figure}

In \cref{sec:output-probabilities-qq-plots}, we also includes Q-Q plot diagnostics of the post-softmax probabilities against two candidate fits. \cref{fig:probs-qq-normal} shows a Q-Q plot versus a Normal distribution, and \cref{fig:probs-qq-t} shows a Q-Q plot versus a Student \(t\) distribution. Neither fit aligns well with the empirical tails, which motivated our choice of heavy-tailed logits in the theoretical framework (see \cref{sec:theoretical-results}).

\subsection{Acceptance recovery}



Based on our findings in \cref{sec:underlying-distributions}, we sample synthetic target distributions from a \(t\) distribution (\cref{fig:target-distribution}) and measure acceptance rates under three schemes: (a) pruning without redistribution, (b) Token-Level Intersection (TLI), and (c) our first-order Taylor approximations of RDK. \cref{fig:acceptance-rates-overlay} visualizes the exact probability-mass of each scheme after pruning to 0.00025\% of the vocabulary (500 tokens from an original 200k). Across a range of pruning levels, RDK maintains an almost constant acceptance rate of 26.7\%, and therefore substantially outperforms both baselines for high pruning rates, as shown in \cref{fig:acceptance-rates-synthetic}. TLI fails to recover acceptance rates because it assigns zero probability to out-of-vocabulary tokens. 
\cref{sec:acceptance-rates-of-off-the-shelf-models} shows the acceptance rates of off-the-shelf models.

\begin{figure}[htbp]
  \centering
  \includegraphics[width=0.8\textwidth]{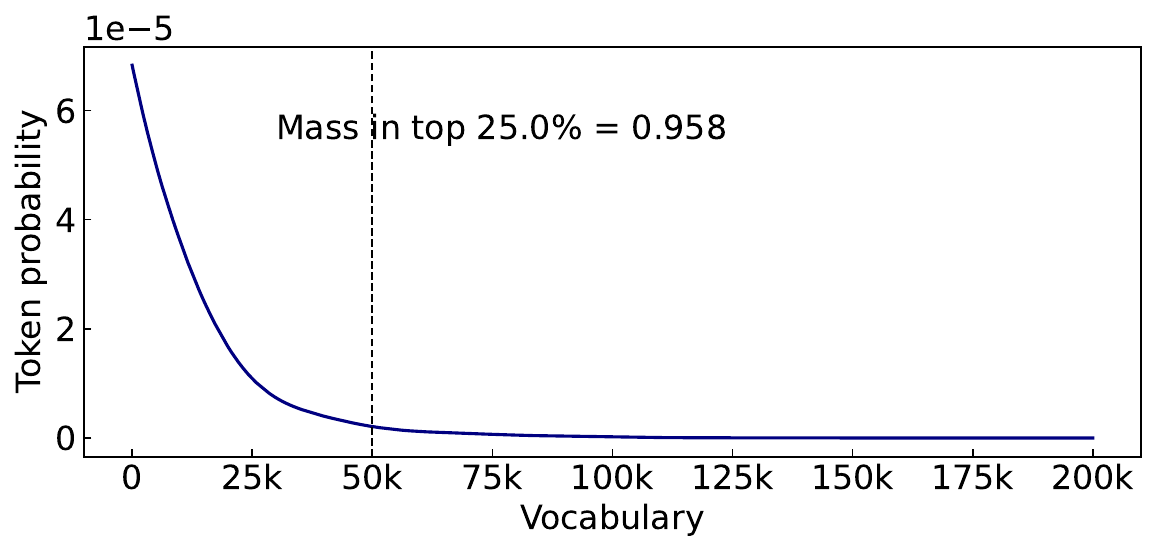}
  \caption{Random target distribution (test sample). The dashed vertical line marks the 25\% quantile, which contains 95.8\% of the total mass. Such concentrated target distributions follow the observed token frequency distribution from \cref{fig:token-frequency} and \citet{zhao2025fr}.}
  \label{fig:target-distribution}
\end{figure}

\begin{figure}[htbp]
  \centering
  \includegraphics[width=0.8\textwidth]{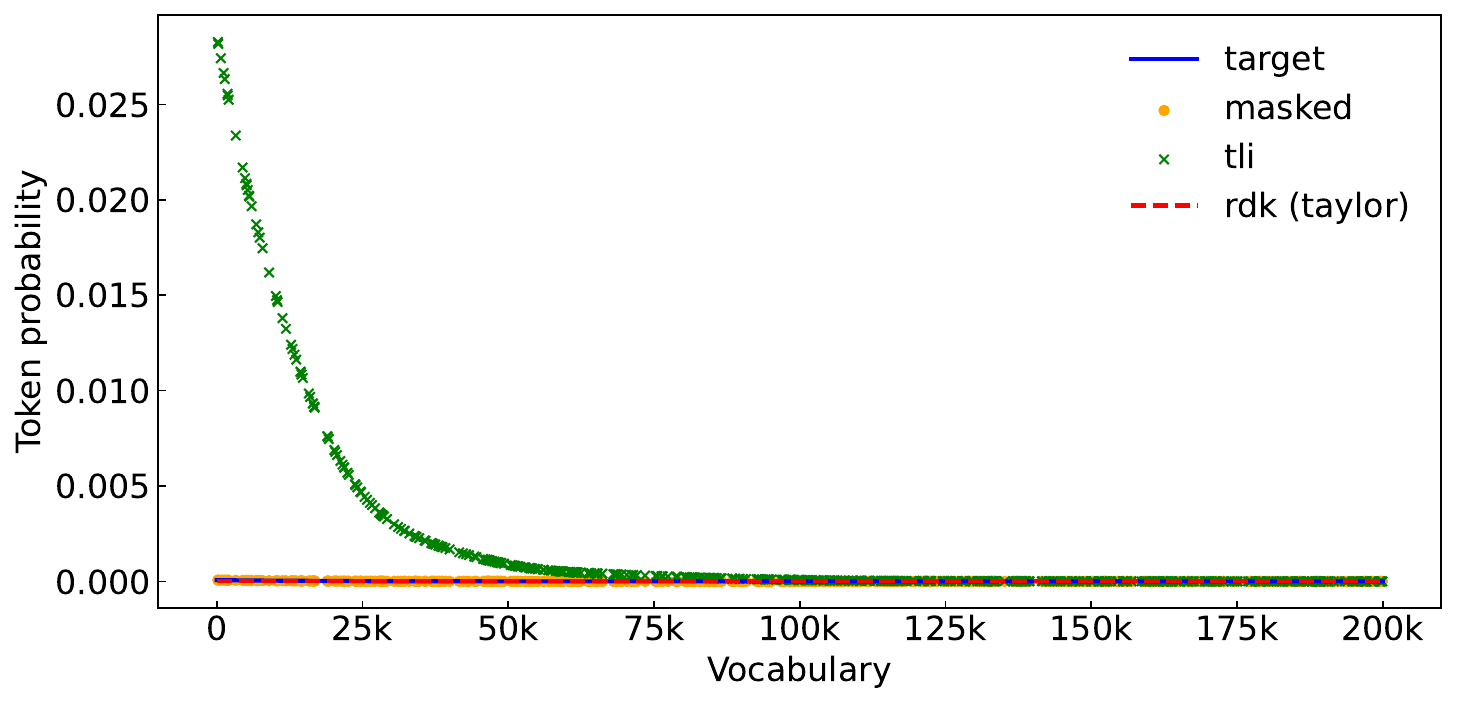}
  \caption{Overlay of probability-mass distributions after pruning to 0.00025\% of the vocabulary (500 tokens from originally 200k). The solid blue line is the true target, orange dots the masked-only baseline, green ×'s TLI, and red dashed the RDK Taylor approximation.}
  \label{fig:acceptance-rates-overlay}
\end{figure}


\begin{figure}[htbp]
    \centering
    \includegraphics[width=0.8\textwidth]{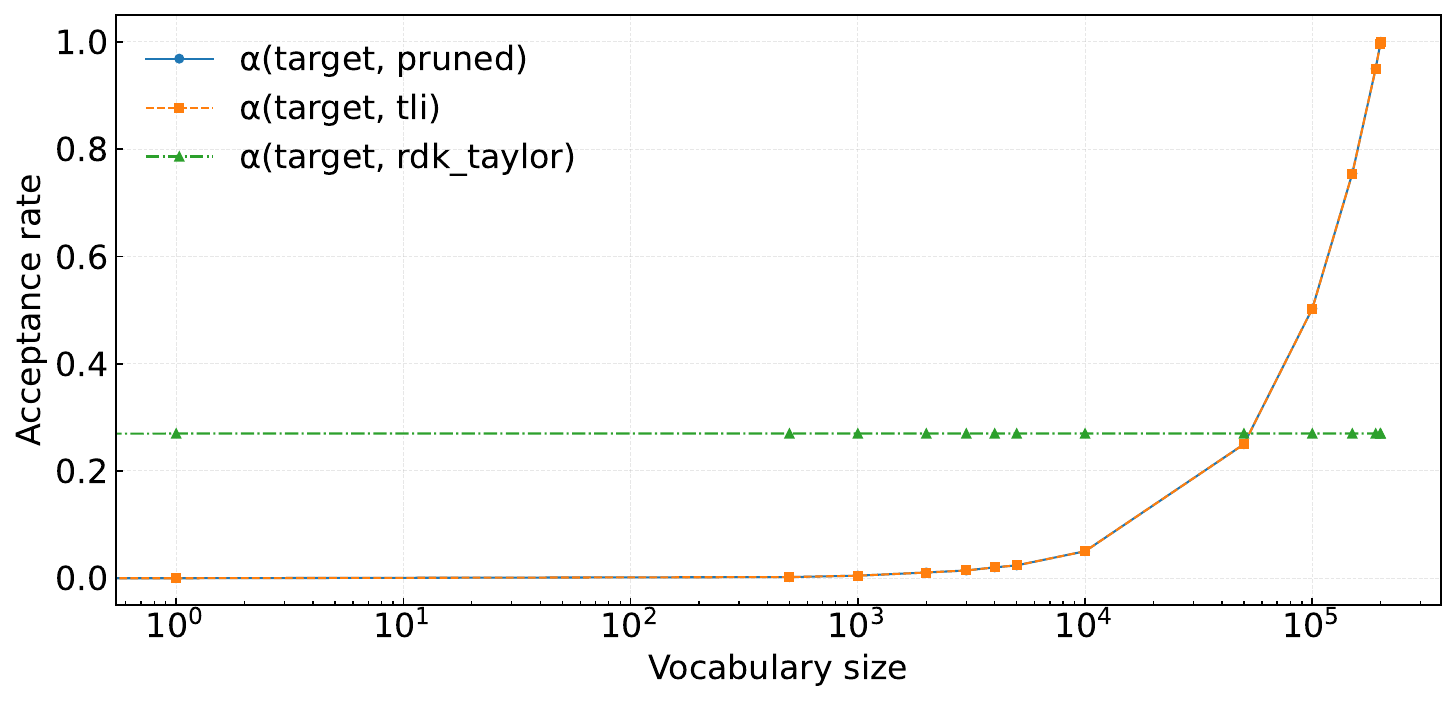}
    \caption{Acceptance rates with synthetic target distributions. TLI fails to recover acceptance rates because it assigns zero probability to out-of-vocabulary tokens. RDK maintains an almost constant acceptance rate of 26.7\%, and therefore substantially outperforms both baselines for high pruning rates.}
    \label{fig:acceptance-rates-synthetic}
\end{figure}

\section{Discussion}\label{sec:discussion}
We introduce the novel concept of out-of-vocabulary (OOV) sampling and present RDK, the first sampler in this category.  
We establish the theoretical foundations for OOV-aware speculative decoding.  
We prove that RDK improves acceptance over both TLI and vanilla samplers.  
We derive a linear-time Taylor approximation.  
This approximation allows efficient scaling to large vocabularies.  

\paragraph{Limitations.}
Estimating the token-affinity matrix $M$ requires one preprocessing pass over sufficiently representative data. This is a one-time cost that can be amortized across all subsequent decoding runs.



\bibliographystyle{plainnat}
\bibliography{bib}

\appendix

\tableofcontents

\section{Existing speculative decoding algorithms}

\begin{algorithm2e}[H]
    \caption{Token-Level Intersection (TLI), rephrased from \citet{timor2025accelerating}}
    \label{alg:tli}
    \KwIn{Probability distributions $p$ and $q$ over vocabularies $T$ and $D$, respectively.}
    \nl Define a probability distribution $q'$ over the vocabulary $T \cap D$ such that $q'(x)=\frac{q(x)}{\sum_{x \in T} q(x)}$ if $x \in T$ and $q'(x) = 0$ otherwise\;
    \nl \textbf{Run} vanilla speculative decoding (\cref{alg:sd}) with $p, q'$\;
\end{algorithm2e}

\begin{algorithm2e}[H]
    \caption{Frequency-Ranked (FR) Pruning, rephrased from \citet{zhao2025fr}}
    \label{alg:fr-prune}
    \KwIn{A drafter of vocabulary $D$ with LM head $W \in \mathbb{R}^{|D| \times h}$, its tokenizer, a training dataset $S$, and the size of the pruned vocabulary $m \le |D|$.}
    \nl Compute the empirical frequency $f(d)$ of each token $d \in D$ over the dataset $S$\;
    \nl Sort the tokens in $D$ by their empirical frequency in descending order\;
    \nl Replace the head matrix of the drafter $W$ with a new head matrix $W' \in \mathbb{R}^{m \times h}$ that corresponds to the $m$ most frequent tokens in $D$ by their empirical frequencies. Construct $W'$ by copying the $m$ rows from $W$ that correspond to the $m$ most frequent tokens, and discarding the rest of the rows\;
    \nl \textbf{Return} the pruned drafter $W'$\;
\end{algorithm2e}

\begin{algorithm2e}[H]
    \caption{Vanilla speculative decoding, rephrased from \citet{leviathan2023fast,chen2023accelerating}}
    \label{alg:sd}
    \KwIn{Probability distributions $p$ and $q$ over a vocabulary $T$. Drafting lookahead $i \in \mathbb{N}$. An input prompt~$c$.}
    \KwOut{A sequence of tokens from $T$, containing between $1$ to $i+1$ tokens.}
    \nl \For{$j \leftarrow 1, \ldots, i$}{
        \nl Sample a draft token from the drafter conditioned on the prompt and previous drafts, $d_j \sim q_{c \oplus d_1 \oplus \ldots \oplus d_{j-1}}$ (where $d_0 := c$)\;
    }
    \nl \label{line:batching} With data parallelism (batching), compute via one target forward pass the $i+1$ logits of the target model conditioned on the prompt and all the draft continuations, $p_{c},~~p_{c \oplus d_1},~~\cdots,~~p_{c \oplus d_1 \oplus \cdots \oplus d_i}$\;
    \nl \For{$j \leftarrow 1, \ldots, i$}{
        \nl Let $x \leftarrow c \oplus d_1 \oplus \cdots \oplus d_{j-1}$ (where $d_0 := c$)\;
        \nl \label{line:reject} \If{$p_{x}(d_j) \le q_{x}(d_j)$}{
            With probability $1 - \frac{p_x(d_j)}{q_x(d_j)}$, \textbf{reject}~$d_j$ and \textbf{goto} line \ref{line:count-accepted-drafts} (i.e., break this for-loop)\;
        }
        \nl \textbf{Accept} the draft token $d_j$\;
    }
    \nl \label{line:count-accepted-drafts} Let $j \in \{0, 1, \ldots, i\}$ be the number of accepted drafts. Set $x \leftarrow c \oplus d_1 \oplus \ldots \oplus d_j$\;
    \nl \eIf{line \ref{line:reject} ever rejected a token}{
        Sample $t \sim r_{x}$ for $r_x(t) := \frac{p_x(t) - \min\{p_x(t), q_x(t)\}}{1 - \sum_{t' \in T}\min\{p_x(t'), q_x(t')\}}$\;
    }{
        Sample $t \sim p_x$\;
    }
    \nl \Return{$d_1, \ldots, d_j, t$}\;
\end{algorithm2e}

\section{Proofs}\label{appendix:proofs}

\subsection{Proof of \cref{lemma:minimizing-kernels-or-l1-increases-acceptance-rate}}\label{appendix:proof-minimizing-kernels}

\MinimizingKernelsOrLOneIncreasesAcceptanceRate*

\begin{proof}
    For any \(x \in \mathcal{V}\), we can write
    \[
    |p(x) - q(x)| = \max\{p(x),q(x)\} - \min\{p(x),q(x)\}.
    \]
    Summing over \(x \in \mathcal{V}\) and using the fact that \(p\) and \(q\) are probability distributions (i.e., \(\sum_x p(x) = \sum_x q(x) = 1\)), we obtain
    \[
    2 = \sum_{x \in \mathcal{V}} \max\{p(x),q(x)\} + \sum_{x \in \mathcal{V}} \min\{p(x),q(x)\}.
    \]
    By \cref{def:acceptance-rate}, the acceptance rate is defined as \(\alpha := \sum_{x \in \mathcal{V}} \min\{p(x),q(x)\}\).
    Then,
    \[
    \sum_{x \in \mathcal{V}} \max\{p(x),q(x)\} = 2 - \alpha.
    \]
    So the \(L_1\) distance is
    \[
    \|p - q\|_1 = \sum_{x \in \mathcal{V}} |p(x) - q(x)| = (2 - \alpha) - \alpha = 2 - 2\alpha.
    \]
    Rearranging gives
    \[
    \alpha = 1 - \frac{1}{2}\|p - q\|_1.
    \]
    Now, recall that the \emph{drafter kernel} is defined as $k(x) := \max\{0,\, q(x)-p(x)\}$.
    For each \(x \in \mathcal{V}\), there are two cases. If \(q(x) \ge p(x)\), then \(k(x)=q(x)-p(x)\) and $|p(x)-q(x)| = q(x)-p(x) = k(x)$.
    Otherwise, we have that \(q(x) < p(x)\), and then \(k(x)=0\) and $|p(x)-q(x)| = p(x)-q(x)$.
    Since \(p\) and \(q\) are both probability distributions, we have 
    \[
    \sum_{x \in \mathcal{V}} \left(q(x) - p(x)\right) = 0.
    \]
    Let 
    \[
    A := \{ x \in \mathcal{V} \mid q(x) > p(x) \} \quad \text{and} \quad B := \{ x \in \mathcal{V} \mid q(x) \le p(x) \}.
    \]
    Then, it follows that
    \[
    \sum_{x \in A} \bigl(q(x)-p(x)\bigr) = -\sum_{x \in B} \bigl(q(x)-p(x)\bigr) = \sum_{x \in B} \bigl(p(x)-q(x)\bigr).
    \]
    Thus, the total \(L_1\) difference is given by
    \[
    \|p-q\|_1 = \sum_{x \in A} \bigl(q(x)-p(x)\bigr) + \sum_{x \in B} \bigl(p(x)-q(x)\bigr) 
    = 2 \sum_{x \in A} \bigl(q(x)-p(x)\bigr).
    \]
    By definition, the \(L_1\) norm of the drafter kernel is
    \[
    \|k\|_1 = \sum_{x \in A} \bigl(q(x)-p(x)\bigr).
    \]
    Hence,
    \[
    \|p-q\|_1 = 2\,\|k\|_1.
    \]
    Substituting this back into the expression 
    \[
    \alpha = 1 - \frac{1}{2}\|p-q\|_1,
    \]
    we obtain
    \[
    \alpha = 1 - \|k\|_1.
    \]
    Thus, minimizing the \(L_1\) norm of the drafter kernel \(\|k\|_1\) increases the acceptance rate \(\alpha\), since \(\alpha\) is inversely proportional to \(\|k\|_1\).
\end{proof}

\subsection{Proof of \cref{thm:rdk-taylor}}\label{appendix:proof-rdk-taylor}

\TaylorRDKApprox*

\begin{proof}
    Write \(q' \coloneqq (q'_i)_{i\in T}\) and \(p_{\mathrm{RDK}} \coloneqq M^\top q'\). From \cref{eq:symbolic-m},
    \[
      M_{ij}
      := \frac{\exp(\Omega_{ij}/\tau)}{\sum_{k\in T}\exp(\Omega_{ik}/\tau)}\,,
    \]
    where \(\Omega\) is the covariance matrix (renamed to avoid conflict). For small \(\Omega_{ij}/\tau\), expand numerator and denominator:
    \[
      \exp(\Omega_{ij}/\tau)
      = 1 + \frac{\Omega_{ij}}{\tau} + O(\|\Omega\|^2),
      \qquad
      \sum_{k\in T}\exp(\Omega_{ik}/\tau)
      = N + s_i + O(\|\Omega\|^2),
    \]
    with 
    \[
      s_i \coloneqq \sum_{k\in T}\frac{\Omega_{ik}}{\tau}.
    \]
    Hence, set 
    \[
      A \coloneqq \frac{\Omega_{ij}}{\tau},
      \qquad
      s_i \coloneqq \sum_{k\in T}\frac{\Omega_{ik}}{\tau},
    \]
    and write \(x \coloneqq s_i/N\).  Then
    \[
      M_{ij}
      = \frac{1 + A}{N + s_i}
      = \frac{1 + A}{N(1 + x)}
      = \frac{1 + A}{N}\,\frac{1}{1 + x}
      = \frac{1 + A}{N}\Bigl(1 - x + O(x^2)\Bigr)
      = \frac{1 + A}{N} - \frac{(1 + A)\,x}{N} + O(\|\Omega\|^2).
    \]
    Since \((1 + A)x = \tfrac{s_i}{N} + O(\|\Omega\|^2)\), we obtain
    \[
      M_{ij}
      = \frac{1 + A}{N} - \frac{s_i}{N^2} + O(\|\Omega\|^2)
      = \frac{1}{N} + \frac{A}{N} - \frac{s_i}{N^2} + O(\|\Omega\|^2).
    \]
    Define
    \[
      \theta_j \coloneqq \sum_{i\in T} q'_i\,\frac{\Omega_{ij}}{\tau},
      \qquad
      \Theta \coloneqq \sum_{i\in T} q'_i\,s_i.
    \]
    Then
    \[
      (p_{\mathrm{RDK}})_j
      = \sum_{i\in T} q'_i\,M_{ij}
      = \frac{1}{N} + \frac{\theta_j}{N} - \frac{\Theta}{N^2} + O(\|\Omega\|^2).
    \]
    Empirically, the covariance matrix is often dominated by a single principal component that aligns closely with the target distribution \(p\). We assume that for each \(i,j\in T\) there is a (small) scalar
    \begin{equation}\label{eq:theta-definition}
      \theta \coloneqq p^\top q'
    \end{equation}
    such that
    \begin{align}\label{eq:theta-approx}
      \frac{\Omega_{ij}}{\tau}
      \approx \theta\,p_j + O(\theta^2),
      \qquad
      s_i \approx \theta + O(\theta^2).
    \end{align}
    Under this assumption the key quantities become
    \[
      \theta_j
      = \sum_{i\in T} q'_i\,\frac{\Omega_{ij}}{\tau}
      \approx \sum_{i\in T} q'_i\bigl(\theta\,p_j + O(\theta^2)\bigr)
      = \theta\,p_j + O(\theta^2),
    \]
    \[
      \Theta
      = \sum_{i\in T}q'_i\,s_i
      \approx \sum_{i\in T}q'_i\bigl(\theta + O(\theta^2)\bigr)
      = \theta + O(\theta^2).
    \]
    Substituting into our earlier expansion gives
    \[
      (p_{\mathrm{RDK}})_j
      = \frac{1}{N}
      + \frac{\theta\,p_j}{N}
      - \frac{\theta}{N^2}
      + O(\theta^2).
    \]
    On the other hand, by definition
    \[
    \tilde p_j \;=\;\frac{N\,q'_j + \theta\,p_j}{\,N + p_j\,}
    = \underbrace{\frac{N\,q'_j}{N + p_j}}_{(1)} \;+\;\underbrace{\frac{\theta\,p_j}{N + p_j}}_{(2)}.
    \]
    We now expand each piece to first order in \(\theta\).

    First, since \(\sum_{i}q'_i=1\) and \(\theta = p^\top q'\), one shows
    \[
    N\,q'_j = 1 + O(\theta).
    \]
    Hence
    \[
    \frac{N\,q'_j}{N + p_j}
    = \frac{1 + O(\theta)}{N + p_j}
    = \frac{1}{N + p_j} + O(\theta).
    \]
    Next, expand \(1/(N+p_j)\) in powers of \(p_j/N\):
    \[
    \frac{1}{N + p_j}
    = \frac{1}{N}\,\frac{1}{1 + \tfrac{p_j}{N}}
    = \frac{1}{N}\Bigl(1 - \tfrac{p_j}{N} + O(N^{-2})\Bigr)
    = \frac{1}{N} - \frac{p_j}{N^2} + O(N^{-3}).
    \]
    Putting these together,
    \[
    (1)\;=\;\frac{N\,q'_j}{N + p_j}
    = \frac{1}{N} - \frac{p_j}{N^2} + O(\theta),
    \]
    and
    \[
    (2)\;=\;\frac{\theta\,p_j}{N + p_j}
    = \theta\,\frac{p_j}{N}\bigl(1 + O(N^{-1})\bigr)
    = \frac{\theta\,p_j}{N} + O(\theta\,N^{-1}).
    \]
    hence
    \[
      \tilde p_j
      = \frac{1}{N}
      + \frac{\theta\,p_j}{N}
      - \frac{p_j}{N^2}
      + O(\theta).
    \]
    Comparing to \((p_{\mathrm{RDK}})_j\), we get
    \[
      \tilde p_j - (p_{\mathrm{RDK}})_j
      = -\frac{p_j}{N^2} + \frac{\theta}{N^2} + O(\theta)
      = \mathcal O\!\Bigl(\theta + \tfrac1{N^2}\Bigr).
    \]
    Hence
    \[
      \|\tilde p - p_{\mathrm{RDK}}\|_1
      = \mathcal O\!\Bigl(\theta + \tfrac1{N^2}\Bigr).
    \]
    Finally, set
    \[
    S \;=\;\sum_{j\in T}\tilde p_j,
    \]
    and write
    \[
    \tilde p = S\,p_{\mathrm{Taylor}}, 
    \qquad
    p_{\mathrm{Taylor}} = \frac{\tilde p}{S}.
    \]
    We have already shown
    \[
    \|\tilde p - p_{\mathrm{RDK}}\|_1 = \mathcal O\!\Bigl(\theta + \tfrac1{N^2}\Bigr).
    \]
    On the other hand, from the expansion
    $\tilde p_j
    = \frac1N + \frac{\theta\,p_j}{N} - \frac{p_j}{N^2} + O(\theta)$,
    summing over \(j\) gives
    \[
    S = \sum_j\tilde p_j
    = 1 + \underbrace{\Bigl(\sum_j O(\theta)\Bigr)}_{=O(\theta)}
    + \underbrace{\Bigl(-\sum_j\frac{p_j}{N^2}\Bigr)}_{=O(1/N^2)}
    = 1 + \delta,
    \quad
    \delta = \mathcal O\!\Bigl(\theta + \tfrac1{N^2}\Bigr).
    \]
    Hence
    \[
    p_{\mathrm{Taylor}}
    = \frac{\tilde p}{1+\delta}
    = \tilde p\,(1 - \delta + O(\delta^2))
    = \tilde p + O(\delta)
    = \tilde p + \mathcal O\!\Bigl(\theta + \tfrac1{N^2}\Bigr).
    \]
    Applying the triangle inequality over
    $p_{\mathrm{Taylor}} - p_{\mathrm{RDK}}
    = \bigl(p_{\mathrm{Taylor}} - \tilde p\bigr)
    + \bigl(\tilde p - p_{\mathrm{RDK}}\bigr)$
    gives
    \[
      \bigl\|p_{\mathrm{Taylor}} - p_{\mathrm{RDK}}\bigr\|_1
      \;\le\;
      \underbrace{\bigl\|p_{\mathrm{Taylor}} - \tilde p\bigr\|_1}
        _{\,\mathcal O(\theta + 1/N^2)\,}
      \;+\;
      \underbrace{\bigl\|\tilde p - p_{\mathrm{RDK}}\bigr\|_1}
        _{\,\mathcal O(\theta + 1/N^2)\,}
      = \mathcal O\!\Bigl(\theta + \tfrac1{N^2}\Bigr).
    \]
    This completes the proof.
\end{proof}

\subsection{Proof of \cref{thm:softmaxstability}}\label{appendix:proof-softmaxstability}

\SoftmaxStability*

\begin{proof}
Define a linear interpolation between the two logits to be $z_s := z + s \eta$ where $s \in [0, 1]$. Let $p_s := \mathrm{softmax}(z_s)$. By the fundamental theorem of calculus, we have that
\[
q - p = p_1 - p_0 = \int_0^1 \frac{d}{ds} p_s \, ds.
\]

The derivative of the Softmax, by the chain rule, is given by
\[
\frac{d}{ds} p_s = \nabla \mathrm{softmax}(z_s) \cdot \eta = J_s \eta,
\]
where \( J_s := \nabla \mathrm{softmax}(z_s) \) is the Jacobian matrix. The entries of \( J_s \) are given by
\[
[J_s]_{i,j} = p_{s_i} (\delta_{i,j} - p_{s_j}),
\]
where \( \delta_{i,j} \) is the Kronecker delta, namely, \( \delta_{i,j} = 1 \) if \( i = j \) and \( \delta_{i,j} = 0 \) otherwise.
By taking the \( L_1 \) norm and applying the triangle inequality, we get
\[
\|q - p\|_1 = \left\| \int_0^1 J_s \eta \, ds \right\|_1 \le \int_0^1 \|J_s \eta\|_1 \, ds \le \int_0^1 \|J_s\|_{1,1} \|\eta\|_1 \, ds.
\]

The induced operator norm of \( J_s \) in \( L_1 \) is the maximum column sum. Hence, we have that
\[
\|J_s\|_{1,1} = \max_j \sum_i |[J_s]_{ij}| = \max_j 2 p_{s_j}(1 - p_{s_j}) \le \frac{1}{2},
\]
since the function \( x(1 - x) \) achieves its maximum at \( x = 1/2 \). By substituting back, we get
\[
\|q - p\|_1 \le \int_0^1 \frac{1}{2} \|\eta\|_1 \, ds = \frac{1}{2} \|\eta\|_1 \le \frac{1}{2} \epsilon.
\]
\end{proof}

\subsection{Proof of \cref{thm:rdk-acceptance}}\label{appendix:proof-rdk-acceptance}

\RDKAcceptance*

\begin{proof}
    We have
    \begin{equation}
        p'=Mq'=
        \begin{bmatrix}
            M_{11}\\
            M_{12}
        \end{bmatrix}.
    \end{equation}
    By \cref{def:acceptance-rate},
    \begin{equation}
        \alpha_\text{TLI}=\min\{p_1,1\}+\min\{p_2,0\}=p_1,
    \end{equation}
    and
    \begin{equation}
        \alpha_\text{RDK}=\min\{p_1,p_1'\}+\min\{p_2,p_2'\}=\min\{p_1,M_{11}\}+\min\{p_2,M_{12}\}.
    \end{equation}

    For a specific instance $(p_1,p_2)$ sampled from random process \cref{def:logits} and softmax, we discuss three cases:
    \textbf{Case 1): $p_1<M_{11}$.}
    \begin{equation}
        \alpha_\text{RDK} = p_1+\min\{p_2,M_{12}\} \geq p_1=\alpha_\text{TLI},
    \end{equation}
    where $\geq$ holds true because of \cref{assumption:m-is-row-stochastic}.

    \textbf{Case 2): $p_1\ge M_{11}$ and $p_2\ge M_{12}$.}
    \begin{equation}
        \alpha_\text{RDK}=M_{11}+M_{12}=1\geq \alpha_\text{TLI}.
    \end{equation}

    \textbf{Case 3): $p_1\ge M_{11}$ and $p_2< M_{12}$.}
    \begin{equation}
        \alpha_\text{RDK}=M_{11}+p_2.
    \end{equation}
    Since $z \sim \mathcal{N}(\mu,\,M)$, we know that with probability 99.46\%, the following two events hold true simultaneously:
    \begin{equation}
    \begin{split}
        &z_1\in [\mu_1-3\sqrt{M_{11}},\mu_1+3\sqrt{M_{11}}],\\
        &z_2\in [\mu_2-3\sqrt{M_{22}},\mu_2+3\sqrt{M_{22}}].
    \end{split}
    \end{equation}
    That is, we have with high probability,
    \begin{equation}
        z_2-z_1\ge \mu_2-3\sqrt{M_{22}}-(\mu_1+3\sqrt{M_{11}}).
    \end{equation}
    Therefore, with high probability,
    \begin{equation}
    \begin{split}
        p_1-p_2&=\frac{\exp(z_1)}{\exp(z_1)+\exp(z_2)}-\frac{\exp(z_2)}{\exp(z_1)+\exp(z_2)}\\
        &=\frac{\exp(z_1)-\exp(z_2)}{\exp(z_1)+\exp(z_2)}\\
        &=\frac{1-\exp(z_2-z_1)}{1+\exp(z_2-z_1)} \\
        &\le \frac{1-\exp(\mu_2-3\sqrt{M_{22}}-(\mu_1+3\sqrt{M_{11}}))}{1+\exp(\mu_2-3\sqrt{M_{22}}-(\mu_1+3\sqrt{M_{11}}))} \quad\left(\text{because $f(\Delta)=\frac{1-\exp(\Delta)}{1+\exp(\Delta)}$ is decreasing}\right)\\
        &=\frac{\exp(\mu_1+3\sqrt{M_{11}})}{\exp(\mu_1+3\sqrt{M_{11}})+\exp(\mu_2-3\sqrt{M_{22}})}-\frac{\exp(\mu_2-3\sqrt{M_{22}})}{\exp(\mu_2+3\sqrt{M_{11}})+\exp(\mu_2-3\sqrt{M_{22}})}\\
        &\le0\quad (\text{because $\mu_1+3\sqrt{M_{11}}\le\mu_2-3\sqrt{M_{22}}$})\\
        &\le M_{11}\quad \text{(because $M_{11}$ is a variance)}.
    \end{split}
    \end{equation}
    This leads to $\alpha_\text{TLI}\leq \alpha_{\text{RDK}}$.
\end{proof}

\subsection{Proof of \cref{theorem:bounded-l1}}\label{appendix:proof-bounded-l1}

\boundedlone*
\begin{proof}
    Assume that $p$ and $q$ are arbitrary probability vectors over an $m$-element event space.
    Let their supports be $T := \{i : p(i) > 0\}$ and $D := \{i : q(i) > 0\}$, respectively.
    Let the intersection and forbidden sets be defined as $\mathcal{I} := T \cap D$ and $\mathcal{F} := D \setminus \mathcal{I}$, respectively.
    Let $q_0$ be the TLI transport of $q$ into $\mathcal{I}$, as introduced in \citet{timor2025accelerating}:
    \[
        q_0(i) =
        \begin{cases}
        \displaystyle \frac{q(i)}{\sum_{j\in \mathcal{I}} q(j)} & \text{if } i\in\mathcal{I},\\
        0 & \text{if } i\notin\mathcal{I}.
        \end{cases}
    \]
    By definition, $M \in \mathbb{R}^{m \times m}$ is a row-stochastic matrix (i.e., the entries of any row sum to $1$), and we have that $q'$ is only supported on $\mathcal{I}$ (i.e., $q'(i) = 0$ for all $i \notin \mathcal{I}$), and $p' := M^\top q'$ is only supported on $T$ (i.e., $p'(i) = 0$ for all $i \notin T$).
    Let $\epsilon$ be a constant such that $\|M^\top p - p\|_1 \le \epsilon$. 
    We will show that $\|p' - p\|_1 \;\le\; \|q_0 - p\|_1 + \epsilon$.
    
    By definition, we have that $p' = M^\top q'$, hence,
    \[
    p' - p = M^\top q' - p
    \]
    By adding and subtracting $M^\top q_0$, we get
    \[
    p' - p = (M^\top q' - M^\top q_0) + (M^\top q_0 - p)
    \]
    By the triangle inequality, we have that
    \[
    \|p' - p\|_1 \le \|M^\top q' - M^\top q_0\|_1 + \|M^\top q_0 - p\|_1
    \]
    Since $M^\top$ is non-expansive (\cref{lemma:non-expansive}), we have that $\|M^\top q' - M^\top q_0\|_1 \le \|q' - q_0\|_1$. Hence, we get
    \[
    \|p' - p\|_1 \le \|q' - q_0\|_1 + \|M^\top q_0 - p\|_1
    \]
    My adding and subtracting $M^\top p$, and then applying the triangle inequality, we get $\|M^\top q_0 - p\|_1 \le \|M^\top q_0 - M^\top p\|_1 + \|M^\top p - p\|_1$. By applying the non-expansiveness of $M^\top$ and the $\epsilon$ assumption, we get $\|M^\top q_0 - p\|_1 \le \|q_0 - p\|_1 + \epsilon$. Therefore, we have that
    \[
    \|p' - p\|_1 \le \|q' - q_0\|_1 + \|q_0 - p\|_1 + \epsilon
    \]
    If we set $q' = q_0$, then we have that $\|q' - q_0\|_1 = 0$, and the result follows.
\end{proof}

\subsection{Proof of \cref{lemma:non-expansive}}

\begin{lemma}\label{lemma:non-expansive}
    Let $M \in \mathbb{R}^{m \times m}$ be a row-stochastic matrix (i.e., the entries of any row are non-negative and sum to $1$). Then $M^\top$ is non-expansive, namely, $\|M^\top x - M^\top y\|_1 \le \|x - y\|_1$ for all $x, y \in \mathbb{R}^m$.
\end{lemma}

\begin{proof}
    Let $x, y \in \mathbb{R}^m$ be arbitrary. We start by writing
    \[
    \|M^\top x - M^\top y\|_1 = \sum_{i=1}^m \left| \sum_{j=1}^m M^\top_{ij}(x_j-y_j) \right|.
    \]
    Noting that $M^\top_{ij} = M_{ji}$, this becomes
    \[
    \|M^\top x - M^\top y\|_1 = \sum_{i=1}^m \left| \sum_{j=1}^m M_{ji}(x_j-y_j) \right|.
    \]
    
    By the triangle inequality, which states that \(\left|\sum_{j=1}^m a_j\right| \le \sum_{j=1}^m |a_j|\), we have \(\left| \sum_{j=1}^m M_{ji}(x_j-y_j) \right| \le \sum_{j=1}^m M_{ji}|x_j-y_j|\). Hence,
    \[
    \|M^\top x - M^\top y\|_1 \le \sum_{i=1}^m \sum_{j=1}^m M_{ji}|x_j-y_j|.
    \]
    
    Since the sums are finite, we can exchange the order of summation, and we get
    \[
    \sum_{i=1}^m \sum_{j=1}^m M_{ji}|x_j-y_j| = \sum_{j=1}^m |x_j-y_j| \left(\sum_{i=1}^m M_{ji}\right).
    \]
    
    Now, because $M$ is row-stochastic, each row of $M$ sums to $1$. Equivalently, for each fixed $j$, the sum of the entries in the $j$-th column of $M^\top$ equals $1$, namely \(\sum_{i=1}^m M_{ji}=1\). Thus,
    \[
    \|M^\top x - M^\top y\|_1 \le \sum_{j=1}^m |x_j-y_j| = \|x-y\|_1.
    \]
\end{proof}

\section{Additional experiments}

\subsection{Logits: Q-Q plot with a normal distribution}\label{sec:logits-qq-normal-distribution}

\cref{fig:logits-qq-normal-distribution} shows the Q-Q plot of the logits distribution with a normal distribution, from which we learn that the logits are more likely to be distributed as the Student \( t \) distribution rather than the normal distribution, as we hypothesized in \cref{sec:theoretical-results}.

\begin{figure}[htbp]
    \centering
    \includegraphics[width=0.8\textwidth]{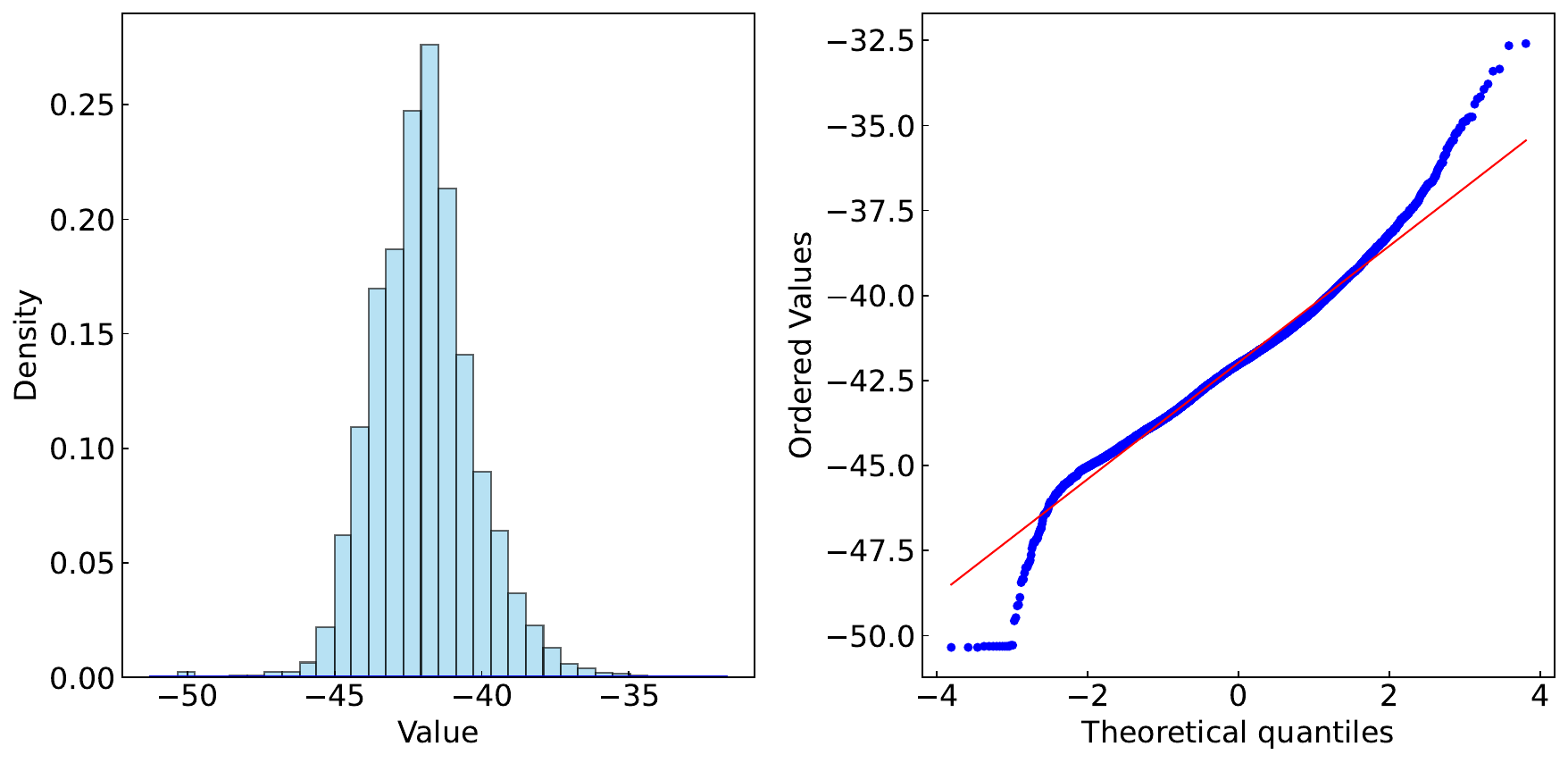}
    \caption{
        (Left) Histogram of observed logits overlaid with the fitted normal PDF. 
        (Right) Q-Q plot comparing empirical logit quantiles (blue) to theoretical normal quantiles (red), showing that a normal does not capture the heavy tails. 
        This confirms that logits follow a Student~\(t\) law more closely than a normal, as hypothesized in \cref{sec:theoretical-results} and confirmed in \cref{fig:logits-qq-t-distribution}.
    }
    \label{fig:logits-qq-normal-distribution}
\end{figure}

\subsection{Output probabilities: Q-Q plots}\label{sec:output-probabilities-qq-plots}

To further assess the adequacy of standard fits, we compare the empirical post-softmax probabilities of \texttt{gpt-2} against a Normal distribution (\cref{fig:probs-qq-normal}) and a Student \(t\) distribution (\cref{fig:probs-qq-t}). In both cases, pronounced deviations in the upper quantiles reveal heavier tails than the Normal and even the \(t\) fit, underscoring our motivation for adopting explicit heavy-tailed priors.

\begin{figure}[htbp]
    \centering
    \includegraphics[width=0.8\textwidth]{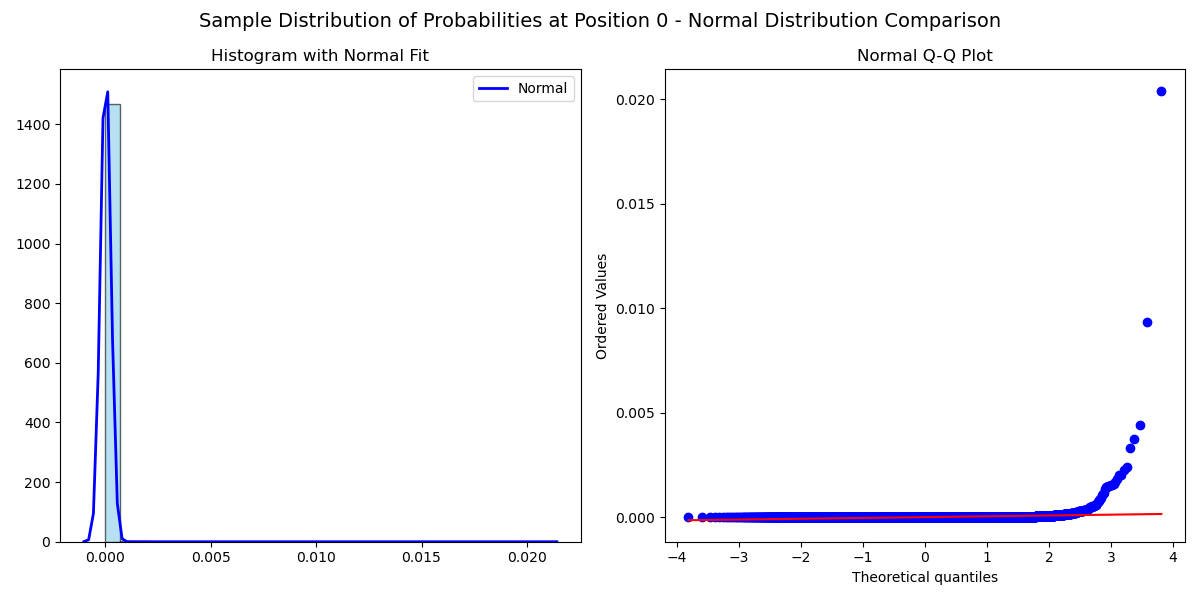}
    \caption{Q-Q plot of \texttt{gpt-2} output probabilities vs.\ a Normal distribution. The systematic departure in the upper tail indicates heavier-tailed behavior than Gaussian.}
    \label{fig:probs-qq-normal}
\end{figure}

\begin{figure}[htbp]
    \centering
    \includegraphics[width=0.8\textwidth]{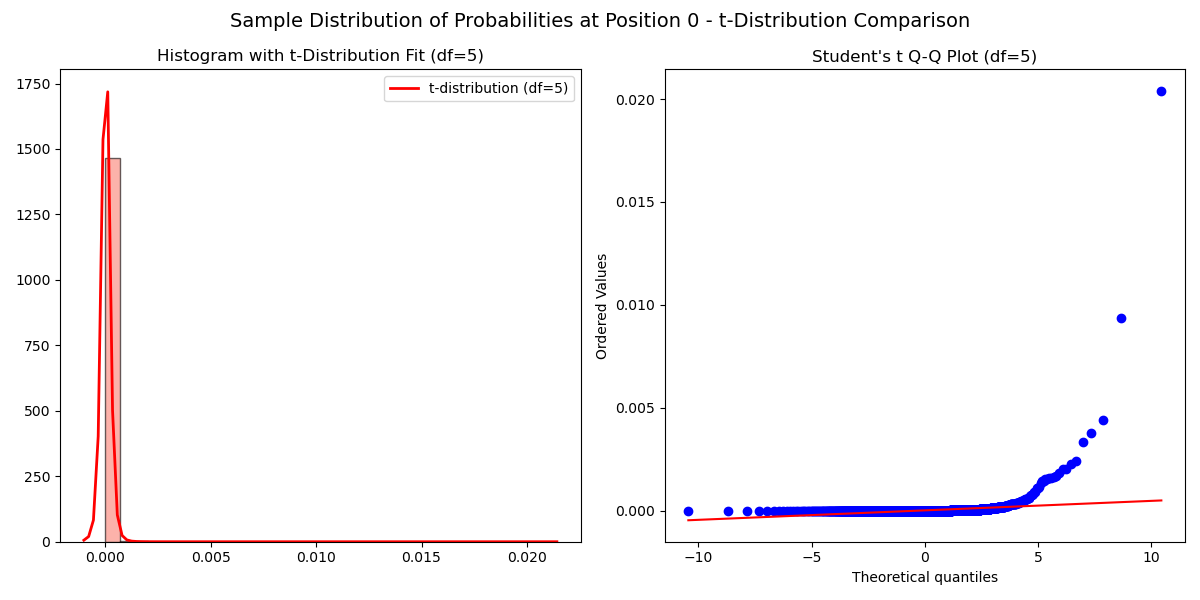}
    \caption{Q-Q plot of \texttt{gpt-2} output probabilities vs.\ a Student \(t\) distribution (df estimated via MLE). Even the \(t\) fit underestimates the extreme quantiles, reinforcing the need for more flexible, heavy-tailed assumptions in our model.}
    \label{fig:probs-qq-t}
\end{figure}

\subsection{Co-occurrence matrix}

To illustrate the prior we use for redistribution, \cref{fig:co-occurrence-matrix} shows a 200\(\times\)200 submatrix of the co-occurrence matrix \(M\) defined in \cref{eq:symbolic-m}.  Each entry \(M_{ij}\) is the normalized covariance between token \(i\) and token \(j\), so rows sum to 1 by construction. The matrix exhibit high sparsity, where the vast majority of entries are near zero, reflecting that most token pairs have negligible co-occurrence.

\begin{figure}[htbp]
    \centering
    \includegraphics[width=0.8\textwidth]{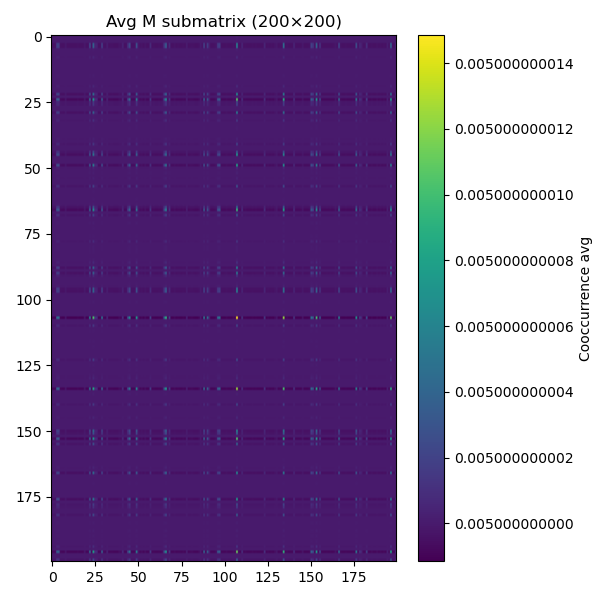}
    \caption{A 200\(\times\)200 block of the co-occurrence matrix \(M\) for our \(t\)-distributed synthetic target.  Most entries are near zero, indicating that token affinities are highly sparse.}
    \label{fig:co-occurrence-matrix}
\end{figure}

\subsection{Acceptance rates of off-the-shelf models}\label{sec:acceptance-rates-of-off-the-shelf-models}

To evaluate the effectiveness of RDK, we measure the acceptance rates over \texttt{Qwen/Qwen3-0.6B-Base}. Under high pruning rates, RDK yields the highest acceptance rate, demonstrating its effectiveness in restoring the target distribution.

\begin{figure}[htbp]
    \centering
    \includegraphics[width=0.8\textwidth]{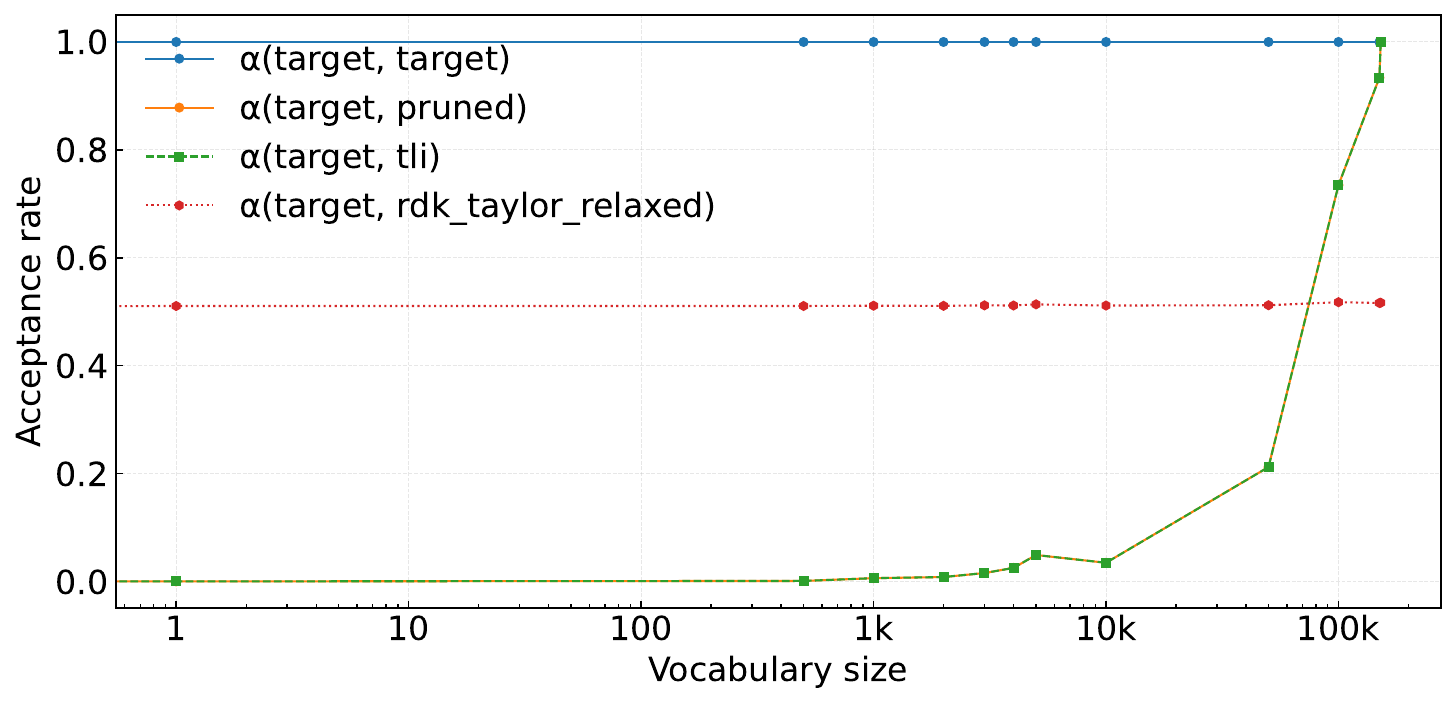}
    \caption{Acceptance rates with off-the-shelf LMs. RDK substantially outperforms both baselines for high pruning rates.}
    \label{fig:acceptance-rates-off-the-shelf}
\end{figure}

\end{document}